\documentclass{article}
\usepackage[utf8]{inputenc}

\title{Improved Bound for Mixing Time of Parallel Tempering}
\newcommand*\samethanks[1][\value{footnote}]{\footnotemark[#1]}
\author{Holden Lee\thanks{Mathematics Department, Duke University. Email: \texttt{holden.lee@duke.edu}, \texttt{zeyu.shen@duke.edu}.} \and Zeyu Shen\samethanks[1]}
\date{}

\usepackage[margin=1.0in]{geometry}
\usepackage{enumitem}
\usepackage{amsmath}
\usepackage{amssymb}
\usepackage{mathtools}
\usepackage{graphicx}
\usepackage{enumitem}
\usepackage[english]{babel}
\usepackage{amsthm}
\usepackage{comment}
\usepackage{float}
\newtheorem{lemma}{Lemma}
\newtheorem{theorem}{Theorem}

\usepackage[ruled,vlined, linesnumbered]{algorithm2e}
\usepackage{systeme}
\newcommand\RR{\mathbb{R}}

\newcommand\Var{\mathbf{Var}}

\usepackage{tikz}
\usetikzlibrary{decorations.pathreplacing}
\definecolor{color1}{rgb}{0.7, 0.2, 0.2}
\definecolor{color2}{rgb}{0.0, 0.4, 0.0}
\definecolor{color3}{rgb}{0.2, 0.4, 0.7}

\usepackage{color}

\begin{document}

\maketitle

\begin{abstract}
    In the field of sampling algorithms, MCMC (Markov Chain Monte Carlo) methods are widely used when direct sampling is not possible. However, multimodality of target distributions often leads to slow convergence and mixing. One common solution is parallel tempering. Though highly effective in practice, theoretical guarantees on its performance are limited. In this paper, we present a new lower bound for parallel tempering on the spectral gap that has a polynomial dependence on all parameters except $\log L$, where $(L + 1)$ is the number of levels. This improves the best existing bound which depends exponentially on the number of modes. Moreover, we complement our result with a hypothetical upper bound on spectral gap that has an exponential dependence on $\log L$, which shows that, in some sense, our bound is tight.
\end{abstract}

\section{Introduction}
A key problem in statistics, computer science, and statistical physics is to draw samples given access to the probability density function, up to a constant of proportionality. Because it is often hard to draw independent samples from the target distribution directly, Markov Chain Monte Carlo (MCMC) methods are often used instead. However, a common difficulty for typical MCMC methods is that for strongly multimodal distributions, MCMC methods take unreasonably long time to reach stationarity. Parallel tempering is an MCMC algorithm that is widely used in sampling from multimodal distributions. Though highly effective in practice, theoretical guarantees on its performance are limited. Since large spectral gap implies fast mixing, a common way to obtain an upper bound on mixing time is to obtain a lower bound on spectral gap. \cite{conditions} discusses conditions for rapid mixing of parallel tempering in sampling from multimodal distributions, and presents a lower bound on the spectral gap that depends exponentially on the number of modes.


In this paper, we present a lower bound on the spectral gap for parallel tempering that has a polynomial dependence on all parameters except $\log L$, where $(L + 1)$ is the number of levels. We show this with a similar canonical path method as in \cite{conditions}, but with a more involved path. This bound has a quasi-polynomial dependence on all parameters, and beats the best existing bound in \cite{conditions} in most cases. We complement our result with a hypothetical upper bound on spectral gap that also has an exponential dependence on $\log L$, which shows that our bound is tight in some sense.

\section{Preliminaries}
Consider a measurable space $(\mathcal{X}, \mathcal{F}, \lambda)$. Usually $\mathcal{X} = \RR^d$ and $\lambda$ is Lebesgue measure, but more general spaces are possible. Suppose we want to draw samples from a distribution $\mu$ on $(\mathcal{X}, \mathcal{F})$, and we do this by simulating a Markov chain with transition kernel $P$, where $\mu$ is the stationary distribution of $P$. If $P$ is $\phi$-irreducible and aperiodic (defined as in \cite{spectralgap}), nonnegative definite and $\mu$-reversible, then this chain converges in distribution to $\mu$ at a rate bounded by the {\em spectral gap} \cite{spectralgap}
$$\textbf{Gap}(P) = \inf_{f \in L_2(\mu)} \left(\frac{\mathcal{E}(f, f)}{\Var_\mu(f)}\right),$$
where $\mathcal{E}(f, f)$ is the Dirichlet form $\langle f, (I - P)f\rangle_\mu$. In particular, for any distribution $\mu_0$ having a density with respect to $\mu$, we have \cite{convergence}
$$\left\Vert \mu_0P^n - \mu\right\Vert_{\text{TV}} \leq \sqrt{\mathcal{X}^2\left(\mu \Vert \mu_0\right)}e^{-n\textbf{Gap}(P)},\:\:\forall n \in \mathbb{N},$$
where $\Vert\cdot \Vert_{\text{TV}}$ is the total variation distance. This explains why large spectral gap implies fast mixing.

\paragraph{Parallel tempering.} Parallel tempering is a classical algorithm that aims to obtain faster mixing in sampling from multimodal distributions. It achieves this by extending the state space with an ``inverse temperature term" $\beta$ that smooths out the distribution to make transitions between different modes easier. When we sample from a distribution $\pi$, each temperature is associated with a stationary distribution $\pi_i$; a natural choice for $\pi_i$ is $\pi^{\beta_i}$, which has fast mixing when $\beta_i$ is small.

The detailed procedure for parallel tempering goes as follows. Suppose we want to sample from a distribution $\pi$. We choose a set of $L + 1$ inverse temperatures $\beta_0 < \cdots < \beta_L = 1$, and construct $L + 1$ levels of chains, with the stationary distribution of the $i^{\text{th}}$ level being $\pi_i$. We run a chain for each value of $\beta$ so that we have $L + 1$ levels of chains running in parallel; we are only interested in the samples at $\beta_L$. In each iteration, we propose swaps of samples between adjacent temperature levels and accept with probability equal to the Metropolis-Hastings ratio. The procedure is formally stated in Algorithm~\ref{pt}.

\begin{algorithm}
\SetAlgoLined
Input: Temperatures $\{\beta_i\}_{i = 0}^L$, Markov chains with stationary distributions $\{\pi_i\}_{i = 0}^L$ and transition kernels $\{T_i\}_{i = 0}^L$, maximum number of iterations $N$\;
Output: Random samples $\{\theta_n\}_{n = 1}^N \subset \mathcal{X}$\;
Initialization: Randomly initialize $\left(\theta_0^0, \ldots, \theta_L^0\right)$\;

\For{$n = 1, \ldots, N$}{
\For{$i = 0, \ldots, L$}{
Propose $\widetilde{\theta_i}$ from $\theta_{i}^{n - 1}$ according to the transition kernel $T_i$\;
Accept $\widetilde{\theta_i}$ with probability $\alpha\left(\widetilde{\theta_i}, \theta_i^{n - 1}\right) = \min\left\{1, \frac{\pi_i\left(\widetilde{\theta_i}\right)}{\pi_i\left(\theta_i^{n - 1}\right)}\right\}$\;
\eIf{\normalfont{Accept}
}{$\theta_i^n \leftarrow \widetilde{\theta_i}$\;}{$\theta_i^n \leftarrow \theta_i^{n - 1}$\;}
}
\For{$i = 1, \ldots, L$}{
Propose swap between the $(i - 1)^{\text{st}}$ chain and the $i^{\text{th}}$ chain\;
Accept swap with probability $\alpha\left(\theta_{i - 1} \leftrightarrow \theta_{i}\right) = \min\left\{1, \frac{\pi_{i - 1}\left(\theta_{i}^n\right)\pi_i\left(\theta_{i - 1}^n\right)}{\pi_i(\theta_{i})\pi_{i - 1}(\theta_{i - 1} )}\right\}$\;
\If{\normalfont{Accept}
}{$\left(\theta_{i - 1}^n, \theta_{i}^n\right) \leftarrow \left\{\theta_{i}^n, \theta_{i - 1}^n\right\}$\;}
}

Add $\theta_L^n$ to the set of samples\;
}
 \caption{Parallel Tempering}
 \nllabel{pt}
\end{algorithm}


\paragraph{Basic Notations.} We use $[n]$ to denote the set $\{1, \ldots, n\}$. Let $\pi$ be the distribution that we want to sample from, which is multimodal with $m$ modes. Suppose we run parallel tempering with $L + 1$ levels of chains. Let $\pi_i$  denote tempered version of $\pi$ with inverse temperature $\beta_i$, where $\pi_i \propto \pi^{\beta_i}$. Let $\mathcal{X}_{pt} = \mathcal{X}^{L + 1}$ denote the state space of the parallel tempering chain. Let $P_{pt}$ denote the transition kernel of the parallel tempering chain. Moreover, given state $\xi = (\xi_0, \ldots, \xi_{k_1},\ldots,\xi_{k_2},\ldots, \xi_L)$, we let $\xi_{[i, k]} = (\xi_0, \ldots, \xi_{i - 1}, k, \xi_{i + 1}, \ldots, \xi_L)$, and $(k_1, k_2)\xi = (\xi_0, \ldots, \xi_{k_2},\ldots,\xi_{k_1},\ldots, \xi_L)$. 

\paragraph{Transition Kernels.} We use $T_i$ to denote the transition kernel at $\beta_i$, which is reversible with respect to $\pi_i$, and use $T$ to denote the transition kernel on $\mathcal{X}_{pt}$. Let $Q$ denote the transition kernel for the swap move. For each update and swap move, we add $\frac{1}{2}$-holding probability to ensure nonnegative definiteness so that the spectral gap translates to mixing time bound. Then, $T$ is given by
$$T\left(\theta, \text{d}\widetilde{\theta}\right) = \frac{1}{2(L + 1)}\sum_{i = 0}^L T_i\left(\theta_i, \text{d}\widetilde{\theta}_i\right)\delta_{\theta_{(-i)}}\left(\widetilde{\theta}_{(-i)}\right)\text{d}\widetilde{\theta}_{(-i)}, \:\:\forall \theta, \widetilde{\theta} \in \mathcal{X}_{pt},$$
where $\theta_{(-i)} = (\theta_0, \ldots, \theta_{i - 1}, \theta_{i + 1}, \ldots, \theta_L)$, $\delta$ is the Dirac's delta function, and often $T_i$ is a Metropolis-Hastings kernel with respect to $\pi_i$; and $Q$ is given by
$$Q(\theta, A) = \frac{1}{2L}\sum_{i = 1}^{L}\textbf{1}_A\bigg((i - 1, i)\theta\bigg)\alpha(\theta_{i - 1}\leftrightarrow \theta_{i}) + \textbf{1}_A(\theta)\left[1 - \frac{1}{2L}\sum_{i = 1}^{L}\alpha(\theta_{i - 1}\leftrightarrow \theta_{i})\right],\:\:\forall A \subset \mathcal{X}_{pt}, \theta \in \mathcal{X}_{pt},$$
where $\alpha(\theta_{i - 1}\leftrightarrow \theta_{i})$ is the probability of accepting a swap between $\theta_{i - 1}$ and $\theta_i$, and $\textbf{1}_A$ is the indicator function of set $A$.

\paragraph{Restriction of Transition kernel to a set.} For transition kernel $P$ reversible with respect to a distribution $\mu$ and any subset $A$ of the state space of $P$, define the restriction of $P$ to $A$ as
$$P\mid_A(\theta, B) = P(\theta, B) + \textbf{1}_B(\theta)P(\theta, A^c), \:\:\forall \theta \in A, B \subset A.$$

\paragraph{Projection of Transition kernel.} Take any partition $\mathcal{A} = \{A_k: k = 1, \ldots, m\}$ of the state space of $P$ such that $\pi(A_k) > 0$ for all $k$, and define the projection matrix of $P$ with respect to $\mathcal{A}$ as
$$\overline{P}(k_1, k_2) = \frac{1}{\pi(A_{k_1})}\int_{A_{k_1}}\int_{A_{k_2}}P(\theta, \text{d}\widetilde{\theta})\mu(\text{d}\theta), \:\:\forall k_1, k_2 \in [m].$$

When $\pi$ is multimodal, $\mathcal{A}$ is often chosen such that $\pi |_{A_k}$ is unimodal for each $k$. As in \cite{zhang}, we consider the projected space $[m]^{L + 1}$ of possible assignments of levels to partition elements. For $\lambda = (\lambda_0, \ldots, \lambda_{L}) \in \mathcal{X}_{pt}$, we define the projection of $\theta$ onto $\mathcal{A}$ as
$$\text{proj}_{\mathcal{A}}(\theta) = \lambda = (\lambda_0, \ldots, \lambda_L),$$
such that
$$\lambda_i = k \text{ if } \theta_i \in A_k,$$
and let $\mathcal{X}_{\lambda} = \{\theta \in \mathcal{X}_{pt} | \text{proj}_{\mathcal{A}}(\theta) = \lambda\}$.

\paragraph{Two important quantities.} Next, we define two important quantities. The first one is
$$\phi = \min_{|i - j| = 1, k\in [m]}\int \frac{\min\{\pi_i(\theta), \pi_j(\theta)\}}{\pi_i(A_k)}\lambda(\text{d}\theta),$$
which is the minimum overlapping volume between two adjacent levels. Intuitively, $\phi$ controls the rate of temperature changes between adjacent levels. Note that for any $k_1, k_2 \in [m]$ and $i \in [L]$, the marginal probability at stationarity of accepting a proposed swap between $\theta_{i - 1} \in A_{k_1}$ and $\theta_i \in A_{k_2}$ is
\begin{equation}
\label{eq555}
\frac{\int_{\xi\in A_{k_1}}\int_{\widetilde{\xi}\in A_{k_2}}\min\left\{\pi_{i - 1}(\xi)\pi_{i}(\widetilde{\xi}), \pi_{i - 1}(\widetilde{\xi})\pi_{i}(\xi)\right\}\lambda(\text{d}\xi)\lambda(\text{d}\widetilde{\xi})}{\pi_{i - 1}(A_{k_1})\pi_{i}(A_{k_2})} \geq \phi^2.
\end{equation}

The second one is
\begin{equation*}
\label{bottleneck}
B = \min_{k \in [m]}\prod_{i = 1}^L \min \left\{1, \frac{\pi_{i - 1}(A_k)}{\pi_i(A_k)}\right\},    
\end{equation*}

which is the ``bottleneck ratio" that arises naturally in the process of swapping.

\paragraph{Tool for bounding spectral gaps.} We present the canonical path method \cite{1991}, which is a method for bounding spectral gaps of finite state space Markov chains. Let $P_1$ and $P_2$ be Markov chain transition matrices on state space $\mathcal{X}$ with $|\mathcal{X}| < \infty$, reversible with respect to densities $\pi_{P_1}$ and $\pi_{P_2}$, respectively. Let $\mathcal{E}_{P_1}$ and $\mathcal{E}_{P_2}$ be the Dirichlet forms of $P_1$ and $P_2$, and let $E_{P_1} = \{(\xi, \widetilde{\xi}):\pi_{P_1}(\xi)P_1(\xi, \widetilde{\xi}) > 0\}$ and $E_{P_2} = \{(\xi, \widetilde{\xi}): \pi_{P_2}(\xi)P_2(\xi, \widetilde{\xi}) > 0\}$ be the edge sets of $P_1$ and $P_2$, respectively. 

\begin{theorem}
\label{canonicalpath}
For each pair $\xi \neq \widetilde{\xi}$ such that $(\xi, \widetilde{\xi}) \in E_{P_2}$, fix a path $\gamma_{\xi, \widetilde{\xi}} = (\xi = \xi_0,\xi_1, \ldots, \xi_\ell = \widetilde{\xi})$ of length $|\gamma_{\xi,\widetilde{\xi}}| = \ell$ such that $(\xi_{s - 1}, \xi_s) \in E_{P_1}$ for $s \in [\ell]$. Define the congestion
$$c = \max_{(\tau, \widetilde{\tau}) \in E_{P_1}}\left\{\frac{1}{\pi_{P_1}(\tau)P_1(\tau, \widetilde{\tau})}\sum_{\gamma_{\xi,\widetilde{\xi}}\ni (\tau, \widetilde{\tau})}|\gamma_{\xi,\widetilde{\xi}}|\pi_{P_2}(\xi)P_2(\xi, \widetilde{\xi})\right\}.$$

Then, we have:
$$\mathcal{E}_{P_2} \leq c \mathcal{E}_{P_1}.$$
\end{theorem}
Note that, in Theorem~\ref{canonicalpath}, different paths can have different lengths.

\section{Lower Bound on the Spectral Gap}
In this section, we present a lower bound on spectral gap for parallel tempering. The bound is formally stated in Theorem~\ref{main}.
\begin{theorem}
\label{main}
Given any partition $\mathcal{A} = \{A_k: k = 1, \ldots, m\}$ of $\mathcal{X}$ such that $\pi_i[A_k] > 0$ for all $i$ and $k$, we have
$$\normalfont{\textbf{Gap}}(P_{pt}) \geq \frac{\phi^2 B^{O(\log L)}}{O(m^3 (L + 1)^{3 + 2\log_23})}\textbf{Gap}(\overline{T}_1)\min_{i, k}\textbf{Gap}(T_i|_{A_k}).$$
\end{theorem}

Note that the bound in \cite{conditions} gives $\textbf{Gap}(P_{pt}) \geq \frac{\phi^2 B^{m + 3}}{2^{12} m^3 (L + 1)^{4}}\textbf{Gap}(\overline{T}_1)\min_{i, k}\textbf{Gap}(T_i|_{A_k})$, which has an exponential dependence on $m$ and is superseded by our exponential dependence on $\log L$ in most cases. Essentially, our improvement stems from the construction of a more involved path for the canonical path method: the congestion of the path depends on the number of samples that have to be moved to a different level; however, the path in \cite{conditions} effectively has to move the samples in every level to a different level. To get a better lower bound, we can actually use a recursive construction, so that only samples in $O(\log L)$ levels are moved to a different level.

To prove Theorem~\ref{main}, we first note that $\overline{P}_{pt}$ is reversible with respect to the probability mass function
\begin{equation}
\label{eq0}
\overline{\pi}(\lambda) \equiv \prod_{i = 0}^L\pi_i(A_{\lambda_i}), \quad\forall \lambda = (\lambda_0, \ldots, \lambda_L)\in [m]^{L + 1}.
\end{equation}


We consider a transition kernel $P_1$ constructed as follows: with probability $\frac{1}{2}$ transition according to $\overline{Q}$, or with probability $\frac{1}{2(L + 1)}$ draw $\lambda_{0}$ according to the distribution $\{\pi_0[A_k]: k \in [m]\}$; otherwise hold. Clearly, $P_1$ is also reversible with respect to $\overline{\pi}$, so $\overline{P}_{pt}$ and $P_1$ have the same stationary distribution. We borrow the following inequality from \cite{conditions}.

\begin{lemma}
\label{lem111}
\normalfont{(Equation (17) in \cite{conditions})} $\normalfont{\textbf{Gap}}(\overline{P}_{pt}) \geq \frac{\textbf{Gap}(P_1)\textbf{Gap}(\overline{T}_1)}{4}.$
\end{lemma}

We now bound $\textbf{Gap}(P_1)$ by comparison with another $\overline{\pi}$-reversible chain using the canonical path method. Define the transition matrix $P_2$ which chooses $i$ uniformly from $\{0, \ldots, L\}$  and then draws $\lambda_i$ according to the distribution
$\pi_i[A_k], k = 1, \ldots, m$.
Intuitively, $P_2$ moves easily between different modes, and thus has a large spectral gap. By comparing $P_1$ and $P_2$, we obtain a lower bound on the spectral gap of $\overline{P}_{pt}$.

Before presenting the path, we first present a procedure that recursively constructs a path $\gamma_{\lambda, (i, j)\lambda}$ from $\lambda$ to $(i, j)\lambda$, where $i < j$; i.e. swaps the samples at the $i^{\text{th}}$ and $j^{\text{th}}$ level. We show that, for any intermediate state in the path, it differs from $\lambda$ in at most $O(\log L)$ levels. The procedure is formally stated in Algorithm~\ref{algo2}.

\begin{algorithm}
\SetAlgoLined
\nllabel{algo2}
Procedure name: \textbf{Swap}($i$, $j$)\;
 \eIf{$j - i \leq 1$}{
   Swap the states at the $i^{\text{th}}$ and $j^{\text{th}}$ level\;
 }
 {
 \textbf{Swap}($i$, $\lfloor \frac{i + j}{2}\rfloor$)\;
 \textbf{Swap}($\lfloor \frac{i + j}{2} \rfloor$, $j$)\;
 \textbf{Swap}($i$, $\lfloor \frac{i +j}{2}\rfloor$)\;
 }
 \caption{Procedure for swapping the states at the $i^{\text{th}}$ and $j^{\text{th}}$ levels.}
\end{algorithm}

\begin{lemma}
For any state $\tau$ in $\gamma_{\lambda, (i, j)\lambda}$, it differs from $\lambda$ in at most $O(\log L)$ levels.
\label{lem1}
\end{lemma}
\begin{proof}
Let $F(\ell)$ denote the maximum number of levels that any intermediate state can differ from the original state during the procedure that swaps the $x^{\text{th}}$ level and the $(x + \ell)^{\text{th}}$ level, where $x + \ell \leq L$. When we call \textbf{Swap}($x$, $x + \lfloor\frac{\ell}{2}\rfloor$) for the first time, any intermediate state can differ from $\lambda$ in at most $F(\lfloor\frac{\ell}{2}\rfloor)$ levels; when we call  \textbf{Swap}($x + \lfloor\frac{\ell}{2}\rfloor$, $x + \ell$), any intermediate state can differ from $\lambda$ in at most $F(\lceil\frac{\ell}{2}\rceil) + 2$ levels; when we call \textbf{Swap}($x$, $x + \lfloor\frac{\ell}{2}\rfloor$) for the second time, any intermediate state can differ from $\lambda$ in at most $F(\lfloor\frac{\ell}{2}\rfloor) + 3$ levels. Thus, we have
$$F(\ell) \leq \max\left\{2 + F\left(\left\lceil\frac{\ell}{2}\right\rceil\right), 3 + F\left(\left\lfloor\frac{\ell}{2}\right\rfloor\right)\right\},$$
which gives $F(\ell) \leq O(\log \ell)$. Therefore, we have that any state in $\gamma_{\lambda, (i, j)\lambda}$ can differ from $\lambda$ in at most $O(\log (j - i)) \leq O(\log L)$ levels.
\end{proof}


Next, fixing some $i \in \{0, \ldots, L\}$ and $k \in [m]$, we present a path $\gamma_{\lambda, \lambda_{[i, k]}}$ between $\lambda$ and $\lambda_{[i, k]}$. Let $k^*$ be the value of $k$ that maximizes $\pi_L(A_k)$. The procedure is formally stated in Algorithm~\ref{algo3}.

\begin{algorithm}
\SetAlgoLined
\nllabel{algo3}
 Change the value at level 0, $\lambda_0$, to $k^*$\;
 
 Call \textbf{Swap}(0, $i$)\;
 
 Change the value at level 0 to $k$\;
 
 Call \textbf{Swap}(0, $i$)\;
 
 Change the value at level 0, $k^*$, to $\lambda_0$\; 
 \caption{Procedure for constructing a path between $\lambda$ and $\lambda_{[i, k]}$.}
\end{algorithm}

To bound $\textbf{Gap}(P_1)$, we first present an upper bound for $\frac{\overline{\pi}(\lambda)P_2(\lambda, \lambda_{[i, k]})}{\overline{\pi}(\tau) P_1(\tau, \widetilde{\tau})}$ for any states $\tau, \widetilde{\tau} \in \gamma_{\lambda, \lambda_{[i, k]}}$.

\begin{lemma}
\label{lem114514}
$\frac{\overline{\pi}(\lambda)P_2(\lambda, \lambda_{[i, k]})}{\overline{\pi}(\tau) P_1(\tau, \widetilde{\tau})} \leq \frac{4m}{\phi^{2}B^{O(\log L)}}$.
\end{lemma}

\begin{proof} 
Observe that
\begin{equation}
\label{awa!}
\overline{\pi}(\lambda)P_2(\lambda, \lambda_{[i, k]}) = \frac{\pi_i(A_k)}{L + 1}\prod_{j = 0}^L \pi_j(A_{\lambda_j}) = \frac{1}{L + 1}\min\{\overline{\pi}(\lambda), \overline{\pi}(\lambda_{[i, k]})\}\max\{\pi_i(A_{\lambda_i}), \pi_i(A_k)\}.
\end{equation}

For any intermediate state $\tau$ in $\gamma_{\lambda, \lambda_{[i, k]}}$, we obtain a lower bound for $\overline{\pi}(\tau)$ in terms of $\min\{\overline{\pi}(\lambda), \overline{\pi}(\lambda_{[i, k]})\}$. Here, we present the following lemma.

\begin{lemma}
For any state $\tau$ in $\gamma_{\lambda, \lambda_{[i, k]}}$, it differs from $\lambda$ and $\lambda_{[i, k]}$ in at most $O(\log L)$ levels.
\end{lemma}
\begin{proof}
When we call \textbf{Swap}(0, $i$) for the first time, any intermediate state can differ from $\lambda$ in at most $O(\log L)$ levels by Lemma~\ref{lem1}. After the first \textbf{Swap}(0, $i$) procedure finishes, the current state differs from $\lambda$ only in the $0^{\text{th}}$ and the $i^{\text{th}}$ levels. Thus, when we call \textbf{Swap}(0, $i$) for the second time, we still have any intermediate state can differ from $\lambda$ in at most $O(\log L)$ levels. Therefore, any state in $\gamma_{\lambda, \lambda_{[i, k]}}$ differs from $\lambda$ in at most $O(\log L)$ levels. Since $\lambda$ and $\lambda_{[i, k]}$ differs in only one level, we also have any state in $\gamma_{\lambda, \lambda_{[i, k]}}$ differs from $\lambda_{[i, k]}$ in at most $O(\log L)$ levels.
\end{proof}

Fix some state $\tau$ from the path, and suppose it differs from $\lambda$ in $n$ levels, where $n = O(\log L)$. 
This actually implies that there are at most $n$ samples that are not in their original levels. 

Suppose all the samples that are not in their original levels are $\{\lambda_{\ell_1}, \ldots, \lambda_{\ell_{n - 1}}, k^*\}$. For $\lambda_{\ell_i}$, suppose its original level is $\ell_i$ and its current level is $\ell_i^*$, and for $k^*$, suppose its current level is $\ell_n^*$. One observation here is that we must have $\ell_i^* < \ell_i$ for all $i \in \{1, \ldots, n - 1\}$ by construction. Thus, we have
\begin{equation*}
\label{wtf}
\frac{\overline{\pi}(\tau)}{\overline{\pi}(\lambda)} = \frac{\pi_{\ell_n^*}(k^*)}{\pi_0(\lambda_0)}\prod_{i = 1}^{n - 1} \frac{\pi_{\ell_i^*}(\lambda_{\ell_i})}{\pi_{\ell_i}(\lambda_{\ell_i})} \geq \frac{B^n}{m} = \frac{B^{O(\log L)}}{m}.
\end{equation*}

Similarly, we have $\frac{\overline{\pi}(\tau)}{\overline{\pi}(\lambda_{[i, k]})} \geq \frac{B^{O(\log L)}}{m}$. Thus, we have
\begin{equation}
\label{orz}
\frac{\overline{\pi}(\tau)}{\min\left(\overline{\pi}(\lambda), \overline{\pi}(\lambda_{[i, k]})\right)} \geq \frac{B^{O(\log L)}}{m}.
\end{equation}

Now, consider some edge $(\tau, \widetilde{\tau})$ on the path $\gamma_{\lambda, \lambda_{[i, k]}}$. If $\widetilde{\tau} = (j, j + 1)\tau$ for some $j$, since the probability of proposing swap according to $Q$ is $\frac{1}{2L}$, we have
    \begin{align}
       \frac{\overline{\pi}(\lambda)P_2(\lambda, \lambda_{[i, k]})}{\overline{\pi}(\tau) P_1(\tau, \widetilde{\tau})} &\leq \frac{2\overline{\pi}(\lambda)P_2(\lambda, \lambda_{[i, k]})}{\overline{\pi}(\tau)\overline{Q}(\tau, \widetilde{\tau})}\leq \frac{4(L + 1)\overline{\pi}(\lambda)P_2(\lambda, \lambda_{[i, k]})}{\overline{\pi}(\tau)\phi^2}\\
       &= \frac{4\min\{\overline{\pi}(\lambda), \overline{\pi}(\lambda_{[i, k]})\}\max\{\pi_i(A_k), \pi_i(A_{\lambda_i})\}}{\overline{\pi}(\tau)\phi^2} \leq \frac{4m}{\phi^2 B^{O(\log L)}}.
       \label{eq-2}
    \end{align}
where the first step is because $P_1$ transitions according to $\overline{Q}$ with probability $\frac{1}{2}$, the second step is by Eq~(\ref{eq555}), the third step is by Eq~(\ref{awa!}), and the fourth step is by Eq~(\ref{orz}). 

If $\widetilde{\tau} = \tau_{[0, k_0]}$ for some $k_0$, then we must be in line 1, 3, or 5 in Algorithm~\ref{algo3}. If we are in line 1, we have $k_0 = k^*$ and $\tau = \lambda$, so
\begin{align*}
   \frac{\overline{\pi}(\lambda)P_2(\lambda, \lambda_{[i, k]})}{\overline{\pi}(\tau) P_1(\tau, \widetilde{\tau})} = \frac{2\overline{\pi}(\lambda)\pi_i(A_k)}{\overline{\pi}(\lambda)\pi_0(A_{k^*})}  \leq \frac{2}{\pi_0(A_{k^*})} \leq \frac{2m}{B},
\end{align*}
where the first step is by definition of $P_1$ and $P_2$, the second step is because $\pi_i(A_k) \leq 1$, and the third step is because $\pi_0(A_{k^*}) \geq \pi_L(A_{k^*})B\geq \frac{B}{m}$. If we are in line 3, we have $k_0 = k$ and $\tau_0 = \lambda_i$, so
    \begin{align}
        \frac{\overline{\pi}(\lambda)P_2(\lambda, \lambda_{[i, k]})}{\overline{\pi}(\tau) P_1(\tau, \widetilde{\tau})} &= \frac{2\overline{\pi}(\lambda)\pi_i(A_k)}{\overline{\pi}(\tau)\pi_0(A_{k})} = \frac{2\min\{\overline{\pi}(\lambda), \overline{\pi}(\lambda_{[i, k]})\}\max\{\pi_i(A_k), \pi_i(A_{\lambda_i})\}}{\min\{\overline{\pi}(\tau), \overline{\pi}(\widetilde{\tau})\}\max\{\pi_0(A_k), \pi_0(A_{\lambda_i})\}}\\
        &\leq \frac{2}{B}\cdot \frac{\min\{\overline{\pi}(\lambda), \overline{\pi}(\lambda_{[i, k]})\}}{\min\{\overline{\pi}(\tau), \overline{\pi}(\widetilde{\tau})\}} \leq \frac{2m}{B^{O(\log L)}},
        \label{eq-1}
    \end{align}
where the second step is by Eq~(\ref{awa!}), the third step is because $\pi_i(A_k) \leq \frac{\pi_0(A_k)}{B} \leq \frac{\max\{\pi_0(A_k), \pi_0(A_{\lambda_i})\}}{B}$ and $\pi_i(A_{\lambda_i}) \leq \frac{\pi_0(A_{\lambda_i})}{B} \leq \frac{\max\{\pi_0(A_k), \pi_0(A_{\lambda_i})\}}{B}$, and the fourth step is by Eq~(\ref{orz}). If we are in line 5, we have $\widetilde{\tau} = \lambda_{[i, k]}$ and $\tau_0 = k^*$, so
$$   \frac{\overline{\pi}(\lambda)P_2(\lambda, \lambda_{[i, k]})}{\overline{\pi}(\tau) P_1(\tau, \widetilde{\tau})} = \frac{2\overline{\pi}(\lambda_{[i, k]})\pi_i(A_{\lambda_i})}{\overline{\pi}(\lambda_{[i, k]})\pi_0(A_{k^*})} \leq \frac{2m}{B}$$
where the first step is because both $P_1$ and $P_2$ are reversible with respect to $\overline{\pi}$. Synthesizing all cases discussed above finishes the proof for Lemma~\ref{lem114514}.
\end{proof}

Next, we present the following lemmas, which provide a bound on the length and congestion of the path.
\begin{lemma}
\label{lem4}
$|\gamma_{\lambda, \lambda_{[i, k]}}| \leq O(L^{\log_23})$.
\end{lemma}
\begin{proof}
Let $F(\ell)$ denote the length of the path produced by Algorithm~\ref{algo2} that swaps the $x^{\text{th}}$ level and the $(x + \ell)^{\text{th}}$ level, where $x + \ell \leq L$. Then, we have
$$F(\ell) = 2F\left(\left\lfloor\frac{\ell}{2}\right\rfloor\right) + F\left(\left\lceil \frac{\ell}{2}\right\rceil\right),$$
which gives 
$$F(2^t) = 3^t, \:F(\ell) \leq F(2^t) = 3^t \leq (2\ell)^{\log_23} \quad\forall t \in \mathbb{N}, 2^{t - 1} \leq \ell < 2^t.$$
i.e. $F(\ell) = O(\ell^{\log_23})$. Thus, $|\gamma_{\lambda, \lambda_{[i, k]}}| = 3 + 2F(i) =  O(i^{\log_23}) \leq O(L^{\log_23})$.
\end{proof}

\begin{lemma}
\label{lem5}
For each given edge $(\tau, \widetilde{\tau})$, there are at most $O(m^2(L + 1)^{1 + \log_23})$ paths going through this edge.
\end{lemma}
\begin{proof}
We first make the following observation. With fixed $i \in \{0, \ldots, L\}, k \in [m]$ and $\lambda, \widetilde{\lambda}$ being two starting states such that $\lambda_{i_0} \neq \widetilde{\lambda}_{i_0}$ for some $i_0 \in [L]$, let $\gamma_{\lambda, \lambda_{[i, k]}}^s$ and $\gamma_{\widetilde{\lambda}, \widetilde{\lambda}_{[i, k]}}^s$ be the $s^{\text{th}}$ edge in each path. Note that $|\gamma_{\lambda, \lambda_{[i, k]}}| = |\gamma_{\widetilde{\lambda}, \widetilde{\lambda}_{[i, k]}}|$, since the length of the path depends only on $i$ by Algorithm~\ref{algo2} and~\ref{algo3}. Then, the main observation is that $\gamma_{\lambda, \lambda_{[i, k]}}^s \neq \gamma_{\widetilde{\lambda}, \widetilde{\lambda}_{[i, k]}}^s$ for each $1 \leq s \leq |\gamma_{\lambda, \lambda_{[i, k]}}|$ because $\lambda$ and $\widetilde{\lambda}$ experience exactly the same procedure in the first $s$ steps to get to $\lambda_{[i, k]}$ and $\widetilde{\lambda}_{[i, k]}$. Thus,
\begin{align*}
    \sum_{\gamma_{\lambda, \lambda_{[i, k]}}\ni (\tau, \widetilde{\tau})} 1 = \sum_{s = 1}^{|\gamma_{\lambda, \lambda_{[i, k]}}|}\sum_{\gamma_{\lambda, \lambda_{[i, k]}}^s =  (\tau, \widetilde{\tau})}1 \leq \sum_{s = 1}^{|\gamma_{\lambda, \lambda_{[i, k]}}|}m \leq O\left(mL^{\log_23}\right),
\end{align*}
i.e. for fixed $i, k$, each edge belongs to at most $O\left(mL^{\log_23}\right)$ paths. Note that the second step is because the condition $\gamma_{\lambda, \lambda_{[i, k]}}^s =  (\tau, \widetilde{\tau})$ uniquely determines $\lambda_{i_0}$ for all $i_0 \in [L]$, and we have the sample on the $0^{\text{th}}$ level unfixed. Since there are $L + 1$ choices for $i$ and $m$ choices for $k$, we can conclude that each edge belongs to at most $O\left(m^2(L + 1)^{1 + \log_23}\right)$ paths.
\end{proof}

Now we go back to prove Theorem~\ref{main}.

\begin{proof}[Proof of Theorem \ref{main}]
By Lemmas~\ref{lem4} and~\ref{lem5}, we have
\begin{equation}
\label{eq999}
\sum_{\gamma_{\lambda, \lambda_{[i, k]}}\ni (\tau, \widetilde{\tau})}|\gamma_{\lambda, \lambda_{[i, k]}}| \leq O\left(m^2(L + 1)^{1 + 2\log_23}\right).
\end{equation}

Since $P_1$ and $P_2$ have the same stationary distribution $\overline{\pi}$, applying Theorem~\ref{canonicalpath}, combined with Lemma~\ref{lem114514} and Eq~(\ref{eq999}), gives
\begin{align*}
\textbf{Gap}(P_2) &\leq \max_{(\tau, \widetilde{\tau}) \in E_{P_1}}\left\{\frac{1}{\pi_{P_1}(\tau)P_1(\tau, \widetilde{\tau})}\sum_{\gamma_{\lambda,\lambda_{[i, k]}}\ni (\tau, \widetilde{\tau})}|\gamma_{\lambda,\lambda_{[i, k]}}|\pi_{P_2}(\lambda)P_2(\lambda, \lambda_{[i, k]})\right\}\textbf{Gap}(P_1)\\
&\leq \max_{(\tau, \widetilde{\tau}) \in E_{P_1}}\left\{\frac{\pi_{P_2}(\lambda)P_2(\lambda, \lambda_{[i, k]})}{\pi_{P_1}(\tau)P_1(\tau, \widetilde{\tau})}\right\}\max_{(\tau, \widetilde{\tau}) \in E_{P_1}}\left\{\sum_{\gamma_{\lambda,\lambda_{[i, k]}}\ni (\tau, \widetilde{\tau})}|\gamma_{\lambda,\lambda_{[i, k]}}|\right\}\textbf{Gap}(P_1)\\
&\leq \frac{O(m^3(L + 1)^{1 + 2\log_23})}{\phi^{2}B^{O(\log L)}}\textbf{Gap}(P_1).
\end{align*}

By Lemma~\ref{lem111}, we get
\begin{equation}
\label{eq1}
    \textbf{Gap}(\overline{P}_{pt}) \geq \frac{\textbf{Gap}(P_1)\textbf{Gap}(\overline{T}_1)}{4} \geq \frac{\phi^2 B^{O(\log L)}}{O(m^3(L + 1)^{2 + 2\log_23})}\textbf{Gap}(\overline{T}_1).
\end{equation}
Here, we borrow the following inequalities from \cite{conditions}.

\begin{lemma}
\label{lem6}
\normalfont{(Equation (14) in \cite{conditions})}
$\normalfont{\textbf{Gap}(P_{pt})} \geq \frac{1}{2}\textbf{Gap}(\overline{P}_{pt})\min_{\lambda \in [m]^{L + 1}}\textbf{Gap}(P_{pt}|_{\mathcal{X}_{\lambda}}).$
\end{lemma}

\begin{lemma}
\label{lem7}
\normalfont{(Equation (15) in \cite{conditions})}
$\normalfont{\textbf{Gap}(P_{pt}|_{\mathcal{X}_{\lambda}})} \geq \frac{1}{8(L + 1)}\min_{i ,k}\textbf{Gap}(T_i|_{A_k}) \quad\forall \lambda \in [m]^{L + 1}.$
\end{lemma}

Combining Lemma~\ref{lem6},~\ref{lem7}, and Eq~(\ref{eq1}) gives
\begin{align*}
\normalfont{\textbf{Gap}}(P_{pt}) \geq\frac{1}{16(L + 1)}\textbf{Gap}(\overline{P}_{pt})\min_{i ,k}\textbf{Gap}(T_i|_{A_k}) \geq \frac{\phi^2 B^{O(\log L)}}{O(m^3(L + 1)^{3 + 2\log_23})}\textbf{Gap}(\overline{T}_1)\min_{i ,k}\textbf{Gap}(T_i|_{A_k}).
\end{align*}
which finishes the proof for Theorem~\ref{main}.
\end{proof}

\section{Hypothetical Upper Bound on Spectral Gap}
In this section, we complement our result with a hypothetical upper bound on spectral gap for parallel tempering. Theorem 5.2 in~\cite{conditions} states that the spectral gap of a chain is upper bounded by the spectral gap of the projected chain. \cite{torpid} discusses sufficient conditions for torpid mixing of parallel tempering, and points out that mixtures of gaussians with unequal variances can mix slowly; however, their slow mixing essentially stems from the fact that the bottleneck ratio in their instance is very small, and their spectral gap can still be upper bounded by $O(B^{O(1)})$. In the sequel, we construct a hypothetical instance and specify the weights of different modes at different temperatures. We only focus on transitions between different modes, but not inside a mode. In this instance, we prove an upper bound of $O\left(B^{O(\log L)}\right)$ on the spectral gap under the assumption that only swappings between adjacent temperature levels are allowed and there's no movement between the modes on the same temperature level, which shows that our lower bound in Theorem~\ref{main} is tight in some sense. Note that we can place the modes far enough from each other, so that there's practically no movement between different modes on the same level. We leave for future works for a more natural example that can lead to a similar upper bound.

\begin{theorem}
\label{upper bound}There exists an instance, in which only swapping between adjacent temperature levels are allowed and there's no movement between the modes on the same level, such that
$$\normalfont{\textbf{Gap}}(P_{pt}) < O\left(B^{O(\log L)}\right).$$
\end{theorem}

Before presenting the proof, we first provide a sketch of our construction. We consider a distribution with $m$ modes and run a parallel tempering chain with $L + 1$ levels on this distribution, where $m = L + 1$. We construct the distribution in such a way that on the $i^{\text{th}}$ level of the chain, the density of the $(i + 1)^{\text{st}}$ mode dominates the density of all the other modes, while the density of the other modes are on the same order of magnitude. Then, we divide the state space into two parts by considering how many samples each state differs from the starting state, show that the probability of transitioning between these two parts is small while the density of each part is large, and use Cheeger's inequality to derive an upper bound on spectral gap. The relative magnitude of modes is illustrated in Figure~\ref{construction}.

\begin{figure}[!h]
\centering
\begin{tikzpicture}[yscale = 0.75, xscale = 0.75]

\node  at (-7.5, 0.1) {$0^{\text{th}}$ level};
\draw [very thick, ->] (-7.5, -1) -- (-7.5, -5);
\node at (-9, -2.4) {Decreasing};
\node at (-9, -3) {Temperature};
\node  at (-7.5, -5.8) {$L^{\text{th}}$ level};

\draw [very thick, color3] (1, 0) circle (12pt);
\node  at (-1, 0) {$\cdots\cdots$};
\draw [very thick, color3] (-3, 0) circle (12pt);
\draw [very thick, color1] (-5, 0) circle (24pt);

\draw [very thick, color3] (1, -2) circle (12pt);
\node at (-1, -2) {$\cdots\cdots$};
\draw [very thick, color1] (-3, -2) circle (24pt);
\draw [very thick, color3] (-5, -2) circle (12pt);

\node  at (-1.3, -3.61) {$\ddots$};
\node  at (-0.75, -4.05) {$\ddots$};
\node  at (-5, -4.05) {$\vdots$};
\node  at (-3, -4.05) {$\vdots$};
\node  at (1, -4.05) {$\vdots$};
\node  at (-5, -3.5) {$\vdots$};
\node  at (-3, -3.5) {$\vdots$};
\node  at (1, -3.5) {$\vdots$};

\draw [very thick, color1] (1, -5.8) circle (24pt);
\node at (-1, -5.8) {$\cdots\cdots$};
\draw [very thick, color3] (-3, -5.8) circle (12pt);
\draw [very thick, color3] (-5, -5.8) circle (12pt);

\end{tikzpicture}
\caption{Illustration of Relative Magnitude of Modes}
\label{construction}
\end{figure}
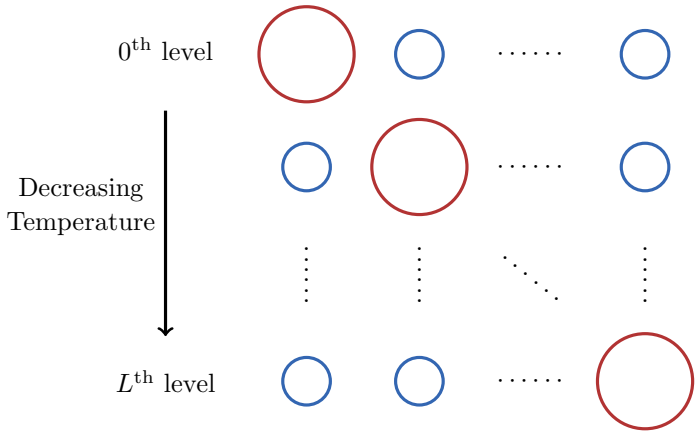

To prove Theorem~\ref{upper bound}, we first present the following lemma.

\begin{lemma}
\label{lem8}
If we want to transform a state from $\lambda$ to $\widetilde{\lambda}$, where $\widetilde{\lambda}_L = \lambda_0$, by only swapping the samples between adjacent temperature levels, there exists some intermediate state $\tau$ which differs from $\lambda$ in $\Omega(\log L)$ levels.
\end{lemma}
\begin{proof}
For any state $\tau$, define $h_\tau(i, j)$ to be the number of levels of $\tau$ between the $i^{\text{th}}$ level and the $j^{\text{th}}$ level, inclusive, that differ from $\lambda$; i.e., $h_\tau(i, j) = |\{\ell: i \leq \ell \leq j, \tau_\ell \neq \lambda_\ell\}|$. To prove this lemma, we prove a stronger statement: assume that there are infinitely many levels (like the $-1^{\text{th}}$ level, $(L + 1)^{\text{st}}$ level, etc., all of which can be swapped to), and $\lambda$ has infinite length; i.e. $\lambda \in [m]^{\mathbb{Z}}$. We apply induction to show that for all $k \in \mathbb{Z}$, if we want to transform a state from $\lambda$ to $\widetilde{\lambda}$, where $\widetilde{\lambda}_{k + L} = \lambda_{i_k}$ for some $i_k \leq k$, then there exists some intermediate state $\tau$ where $h_\tau(k + 1, k + L) \geq \lfloor \log_2 L\rfloor$.

The base case $L = 1$ clearly holds. Suppose the claim holds for $L$, we prove that it also holds for $2L$. Suppose we want to transform a state from $\lambda$ to $\widetilde{\lambda}$, where $\widetilde{\lambda}_{k + 2L} = \lambda_{i_k}$ for some $i_k \leq k$. We first present the following lemma.

\begin{lemma}
\label{lem9}
We only need to consider the case where, in the process of transforming a state from $\lambda$ to $\widetilde{\lambda}$, where $\widetilde{\lambda}_{k + 2L} = \lambda_{i_k}$ for some $i_k \leq k$, there exists some intermediate state $\tau^1$ where $h_{\tau^1}(k + L + 1, k + 2L) \geq \lfloor \log_2L\rfloor$, and $\lambda_{i_k}$ is swapped to somewhere between the $(k + L + 1)^{\text{st}}$ level and the $(k + 2L)^{\text{th}}$ level; i.e., $\tau^1_\ell = \lambda_{i_k}$ for some $k + L + 1 \leq \ell \leq k + 2L$.
\end{lemma}
\begin{proof}
Clearly, there exists an intermediate state $\tau$ such that $h_\tau(k + L + 1, k + 2L) \geq \lfloor \log_2L\rfloor$. We show that, in addition to this, we can either find such an intermediate state $\tau^1$ such that $\lambda_{i_k}$ is swapped to somewhere between the $(k + L + 1)^{\text{st}}$ level and the $(k + 2L)^{\text{th}}$ level or conclude the induction. We take $\tau^1$ to be the last intermediate state $\tau$ where $h_\tau(k + L + 1, k + 2L) \geq \lfloor \log_2L\rfloor$. If $\lambda_{i_k}$ is not between the $(k + L + 1)^{\text{st}}$ level and the $(k + 2L)^{\text{th}}$ level at this point, then we argue that we only need to consider the case where it is at some level $i_k' \leq k$. Actually, we can observe that if we have $h_{\tau^1}(k + 1, k + L) > 0$, then, since $h_{\tau^1}(k + L + 1, k + 2L) \geq \lfloor \log_2L\rfloor$, we have $h_{\tau^1}(k + 1, k + 2L) \geq \lfloor\log_2L\rfloor + 1 = \lfloor \log_22L\rfloor$, which will conclude the induction. Thus, we may assume that $h_{\tau^1}(k + 1, k + L) = 0$. Since $\lambda_{i_k}$ hasn't reached the $(k + 2L)^{\text{th}}$ level and it's not between the $(k + L + 1)^{\text{st}}$ level and the $(k + 2L)^{\text{th}}$ level by assumption, this implies $i_k' \leq k$. Then, by induction hypothesis, we would encounter a later intermediate state $\tau^2$ where $h_{\tau^2}(k + 1, k + L) \geq \lfloor \log_2L\rfloor$. Suppose $\lambda_{i_k}$ is at level $i_k''$ in $\tau^2$, where $i_k'' \leq k + L$. Here, we only need to consider the case where $h_{\tau^2}(k + L + 1, k + 2L) = 0$; otherwise the induction similarly concludes. 
Again by induction hypothesis, in order to swap $\lambda_{i_k}$ from level $i_k''$ up to the $(k + 2L)^{\text{th}}$ level, there exists a later intermediate state $\tau^3$ with $h_{\tau^3}(k + L + 1, k + 2L) \geq \lfloor \log_2L\rfloor$, which contradicts the fact that $\tau^1$ is the last intermediate state $\tau$ where $h_\tau(k + L + 1, k + 2L) \geq \lfloor \log_2L\rfloor$. This concludes the proof of Lemma~\ref{lem9}.
\end{proof}

We say a sample $s$ moves across levels $(k + 1) \sim (k + 2L)$ if it moves from level $i_s > k + 2L$ to level $i_s' \leq k + 1$ or from level $i_s < k + 1$ to level $i_s' \geq k + 2L$. One observation here is that we can assume that $\lambda_{i_k}$ is the first sample that moves across levels $(k + 1) \sim (k + 2L)$. This is because we can otherwise take the first sample that moves across levels $(k + 1) \sim (k + 2L)$ and apply the same argument, possibly reversing the levels. We proceed with our proof under this assumption.

 By Lemma~\ref{lem9}, we only need to consider the case where there exists some intermediate state $\tau^1$ such that $h_{\tau^1}(k + L + 1, k + 2L) \geq \lfloor \log_2L\rfloor$ and $\lambda_{i_k}$ is between the $(k + L + 1)^{\text{st}}$ level and the $(k + 2L)^{\text{th}}$ level in $\tau^1$. If $h_{\tau^1}(k + 1, k + L) > 0$, then the induction concludes. Thus, we may assume that $h_{\tau^1}(k + 1, k + L) = 0$. Since $\lambda_{i_k}$ is between the $(k + L + 1)^{\text{st}}$ level and the $(k + 2L)^{\text{th}}$ level in $\tau^1$, and it was originally at level $i_k \leq k$, there exists some sample $s$ which was originally at level $i_s > k + L$ that is now at level $i_s' \leq k + L$. Since $h_{\tau^1}(k + 1, k + L) = 0$, we have $i_s' \leq k$. Also, as we have assumed that $\lambda_{i_k}$ is the first sample that moves across levels $(k + 1) \sim (k + 2L)$, we also have $i_s \leq k + 2L$; i.e., $s$ was originally between the $(k + L + 1)^{\text{st}}$ level and the $(k + 2L)^{\text{th}}$ level. Since $s$ ultimately reaches level $i_s' \leq k$, there must be a point when it reaches the $(k + 1)^{\text{st}}$ level. Then, by Lemma~\ref{lem9} with the levels reversed, we only need to consider the case where there exists some intermediate state $\tau^4$ where $h_{\tau^4}(k + 1, k + L) \geq \lfloor \log_2L\rfloor$, and $s$ is between the $(k + 1)^{\text{st}}$ level and the $(k + L)^{\text{th}}$ level. At this point, since $s$ is not at its original level, we have $h_{\tau^4}(k + 1, k + 2L) \geq \lfloor \log_2L\rfloor + 1 = \lfloor \log_2 2L\rfloor$, again concluding the induction.

By now, the induction is complete. To complete the proof of Lemma~\ref{lem8}, we observe that if we want to move a sample from the $k^{\text{th}}$ level to the $(k + L)^{\text{th}}$ level, and define
$$f(L) = \min_{\text{all paths }\gamma}\max_{\tau\in\gamma} h_\tau(k + 1, k + L),$$
we clearly have $f(L + 1) \geq f(L)$, since on one hand, during the process of swapping $\lambda_k$ to the $(k + L + 1)^{\text{st}}$ level, we will always pass through the $(k + L)^{\text{th}}$ level, and on the other hand, for any intermediate state $\tau$, we always have $h_\tau(k + 1, k + L) \leq h_\tau(k + 1, k + L + 1)$ by definition. Note that $f$ is still defined under the assumption that there are infinitely many levels. Combined with our inductive step, which proceeds from $L$ to $2L$, we can conclude that $f(L) \geq \lfloor \log_2L\rfloor$ for all $L \in \mathbb{N}$.

Going back to our original claim where there are only $(L + 1)$ levels, since a path in a chain with finite number of levels is always a valid path in a chain with infinitely many levels, our original claim holds true.
\end{proof}

\paragraph{Construction.} We consider the scenario where $L + 1 = m$ and $\beta_i = \frac{i + 1}{L + 1}$. Pick $m$ well-separated points $\{x_k\}_{k = 1}^m$ which will be used as the centers of the modes. Suppose we want to sample from $\pi$; we present a construction for $\pi^{1 / (L + 1)}$ as follows:
$$\pi^{1/(L + 1)} = \frac{1}{Z_0}\left(w_k\pi_k + \sum_{k = 1}^m \sum_{r = 0}^{L} w_{k r}\pi_{kr}\right),$$
where for each fixed $k$, $\pi_k$ and $\pi_{kr}$ for each $r \in \{0, \ldots, L\}$ are unnormalized uniform distributions around $x_k$ with unit density in a space with volume $V_k$ and $V_{kr}$, which are disjoint with each other, and $Z_0$ is the normalizing constant. Assume that $\{x_k\}_{k = 1}^m$ are far enough from each other so that we can approximate $\pi^{(i + 1) / (L + 1)}$ with $\frac{1}{Z_i}\left(w_k^{i + 1}\pi_k^{i + 1} + \sum_{k = 1}^m \sum_{r = 0}^{L} w_{k r}^{i + 1}\pi_{kr}^{i + 1}\right)$, where $Z_i$ is its normalizing constant; i.e., cross terms can be safely ignored. Consider the partition $\mathcal{A} = \{A_k: k = 1, \ldots, m\}$ of $\mathcal{X}$ such that $A_k$ contains the mode with center $x_k$. Let $\gamma = (L + 1)^3$ and
$$w_{k} = \gamma^{2k}, \quad V_{k} = \gamma^{-k^2 + 1},$$
$$w_{kr} = \gamma^{2r + 1}, \quad V_{kr} = \gamma^{-r^2 - r} \quad\forall r \in \{0, \ldots, L\}.$$

Let $\overline{\pi}$ be as defined in Eq~(\ref{eq0}). We show that this instance yields the desired upper bound. Let $\pi_i$ be the normalized form of $\pi^{(i + 1)/(L + 1)}$, and $Z_i$ be the normalizing constant. We first present the following lemma, which provides bounds on the density of each mode.

\begin{lemma}
\label{lem11}
$\pi_i(A_{i + 1}) > 1 - \frac{1}{L + 1} \text{ and } \frac{1}{2(L + 1)^3} < \pi_i(A_k) < \frac{1}{(L + 1)^2} \quad\forall i \in \{0, \ldots, L\} \text{ and } k \neq i + 1$.
\end{lemma}

\begin{proof}
Fix some level $i \in \{0, \ldots, L\}$. We compute
$$\pi_i(A_k) = \frac{1}{Z_i}\left(\gamma^{2(i + 1)k - k^2 + 1} + \sum_{r = 0}^L \gamma^{(2r + 1)(i + 1) - r^2 - r}\right) \quad \forall k \in [m].$$

Since $2(i + 1)k - k^2 + 1 \leq (i + 1)^2 + 1$ and achieves equality when $k = i + 1$, we have
$$\gamma^{2(i + 1)k - k^2 + 1} \leq \gamma^{(i + 1)^2 + 1} \text{ and achieves equality when } k = i + 1,$$
and since $(2r + 1)(i + 1) - r^2 - r \leq (i + 1)^2$ and achieves equality when $r = i$ or $r = i + 1$, we have
$$\gamma^{(2r + 1)(i + 1) - r^2 - r} \leq \gamma^{(i + 1)^2} \text{ and achieves equality when } r = i \text{ or }r = i + 1.$$
Therefore, 
\begin{align*}
\pi_i(A_{i + 1}) &= \frac{\frac{1}{Z_i}\left(\gamma^{(i + 1)^2 + 1} + \sum_{r = 0}^L \gamma^{(2r + 1)(i + 1) - r^2 - r}\right)}{\sum_{k = 1}^m\left[\frac{1}{Z_i}\left(\gamma^{2(i + 1)k - k^2 + 1} + \sum_{r = 0}^L \gamma^{(2r + 1)(i + 1) - r^2 - r}\right)\right]}\\ &> \frac{\gamma^{(i + 1)^2 + 1}}{\gamma^{(i + 1)^2 + 1} + (L + 1)^2\gamma^{(i + 1)^2}}  = \frac{(L + 1)^3}{(L + 1)^3 + (L + 1)^2} >1 - \frac{1}{L + 1},
\end{align*}
where the second step is because 
$$\gamma^{(i + 1)^2 + 1} + \sum_{r = 0}^L \gamma^{(2r + 1)(i + 1) - r^2 - r} > \gamma^{(i + 1)^2 + 1}$$ 
and 
\begin{align*}
\sum_{k = 1}^m\left(\gamma^{2(i + 1)k - k^2 + 1} + \sum_{r = 0}^L \gamma^{(2r + 1)(i + 1) - r^2 - r}\right) &< \gamma^{(i + 1)^2 + 1} + \sum_{k = 1}^m\sum_{r = 0}^L\gamma^{(i + 1)^2}\\ &= \gamma^{(i + 1)^2 + 1} + (L + 1)^2\gamma^{(i + 1)^2};
\end{align*}
here, we used the fact that $m = L + 1$. For $k \neq i + 1$, we similarly have
\begin{align*}
\pi_i(A_k) &= \frac{\frac{1}{Z_{i}}\left(\gamma^{2(i + 1)k - k^2 + 1} + \sum_{r = 0}^L \gamma^{(2r + 1)(i + 1) - r^2 - r}\right)}{\sum_{h = 1}^m\left[\frac{1}{Z_{i}}\left(\gamma^{2(i + 1)h - h^2 + 1} + \sum_{r = 0}^L \gamma^{(2r + 1)(i + 1) - r^2 - r}\right)\right]}\\ &> \frac{\gamma^{(i + 1)^2}}{\gamma^{(i + 1)^2 + 1} + (L + 1)^2\gamma^{(i + 1)^2}} = \frac{1}{(L + 1)^3 + (L + 1)^2} > \frac{1}{2(L + 1)^3}
\end{align*}
and
$$\pi_i(A_k) < \frac{(L + 1)\gamma^{(i + 1)^2}}{\gamma^{(i + 1)^2 + 1} + (L + 1)\gamma^{(i + 1)^2}} = \frac{L + 1}{(L + 1)^3 + (L + 1)} < \frac{1}{(L + 1)^2}.$$
This finishes the proof of Lemma~\ref{lem11}.
\end{proof}

Next, we provide a lower bound on the bottleneck ratio.

\begin{lemma}
\label{sb}
$B > \frac{1}{(L + 1)^7}$, where $B$ is the bottleneck ratio of this chain.
\end{lemma}
\begin{proof}
Fixing $i \in [L]$ and $k \neq i, i + 1 \text{ or }i + 2$, we first provide a lower bound on $\frac{\pi_{i - 1}(A_k)}{\pi_i(A_k)}$. We compute
\begin{align*}
    \pi_i(A_k) &= \frac{\gamma^{2(i + 1)k - k^2 + 1} + \sum_{r = 0}^L \gamma^{(2r + 1)(i + 1) - r^2 - r}}{\sum_{h = 1}^m\left(\gamma^{2(i + 1)h - h^2 + 1} + \sum_{r = 0}^L \gamma^{(2r + 1)(i + 1) - r^2 - r}\right)}\\
    &= \frac{\gamma^{2(i + 1)k - k^2 + 1} + \sum_{r = 0}^L \gamma^{(2r + 1)(i + 1) - r^2 - r}}{\sum_{r = 1}^{L + 1}\gamma^{2(i + 1)r - r^2 + 1} + (L + 1)\sum_{r = 0}^L \gamma^{(2r + 1)(i + 1) - r^2 - r}}\\
    &= \frac{\gamma^{1 - (i + 1 - k)^2 }+ \sum_{r = 0}^L \gamma^{-(i + 1 - r)(i - r)}}{\sum_{r = 1}^{L + 1}\gamma^{1 - (i + 1 - r)^2} + (L + 1)\sum_{r = 0}^L\gamma^{-(i + 1 - r)(i - r)}},
\end{align*}
and
$$\pi_{i - 1}(A_k) = \frac{\gamma^{1 - (i - k)^2} + \sum_{r = 0}^L\gamma^{-(i - r)(i - 1 - r)}}{\sum_{r = 1}^{L + 1}\gamma^{1 - (i - r)^2} + (L + 1)\sum_{r = 0}^L\gamma^{-(i - r)(i - 1 - r)}}.$$

Because $k \neq i, i + 1 \text{ or }i + 2$, $\gamma^{1 - (i + 1 - k)^2} \leq \frac{1}{\gamma} = \frac{1}{(L + 1)^3}$, and $\gamma^{-(i + 1 - r)(i - r)} = 1$ when $r = i$ or $i + 1$ and $\gamma^{-(i + 1 - r)(i - r)} \leq \frac{1}{\gamma} = \frac{1}{(L + 1)^3}$ otherwise, we have
\begin{equation}
\label{eq01}
    \gamma^{1 - (i + 1 - k)^2 }+ \sum_{r = 0}^L \gamma^{-(i + 1 - r)(i - r)} < 2 + (L + 1)\frac{1}{(L + 1)^3} = 2 + \frac{1}{(L + 1)^2}.
\end{equation}
Also, because $\gamma^{-(i - r)(i - 1 - r)} = 1$ when $r = i$ or $i - 1$, we have
\begin{equation}
\label{eq02}
    \gamma^{1 - (i - k)^2} + \sum_{r = 0}^L\gamma^{-(i - r)(i - 1 - r)} > 2,
\end{equation}
which yields
\begin{equation*}
\label{eq000}
   \frac{\gamma^{1 - (i - k)^2} + \sum_{r = 0}^L\gamma^{-(i - r)(i - 1 - r)}}{\gamma^{1 - (i + 1 - k)^2 }+ \sum_{r = 0}^L \gamma^{-(i + 1 - r)(i - r)}} > \frac{2}{2 + \frac{1}{(L + 1)^2}} > 1 - \frac{1}{(L + 1)^2}. 
\end{equation*}

On the other hand, because $\gamma^{1 - (i - r)^2} = (L + 1)^3$ when $r = i$ and $\gamma^{1 - (i - r)^2} \leq 1$ otherwise, we have
$$\sum_{r = 1}^{L + 1}\gamma^{1 - (i + 1 - r)^2} + (L + 1)\sum_{r = 0}^L\gamma^{-(i + 1 - r)(i - r)} < (L + 1)^3 + L + 2(L + 1) + (L + 1)^2\frac{1}{(L + 1)^3} < (L + 1)^3 + 3(L + 1),$$
$$\sum_{r = 1}^{L + 1}\gamma^{1 - (i - r)^2} + (L + 1)\sum_{r = 0}^L\gamma^{-(i - r)(i - 1 - r)} > (L + 1)^3 + 2(L + 1),$$
which yields
\begin{equation}
\label{eq03}
   \frac{\sum_{r = 1}^{L + 1}\gamma^{1 - (i + 1 - r)^2} + (L + 1)\sum_{r = 0}^L\gamma^{-(i + 1 - r)(i - r)}}{\sum_{r = 1}^{L + 1}\gamma^{1 - (i - r)^2} + (L + 1)\sum_{r = 0}^L\gamma^{-(i - r)(i - 1 - r)}} > \frac{(L + 1)^3 + 2(L + 1)}{(L + 1)^3 + 3(L + 1)} > 1 - \frac{1}{(L + 1)^2} .
\end{equation}

Therefore, we have
$$\frac{\pi_{i - 1}(A_k)}{\pi_i(A_k)} > \left(1 - \frac{1}{(L + 1)^2}\right)^2 > 1 - \frac{1}{L + 1}.$$

When $k = i$, by Lemma~\ref{lem11} we have
$$\frac{\pi_{i - 1}(A_k)}{\pi_i(A_k)} > \frac{1 - \frac{1}{L + 1}}{\frac{1}{(L + 1)^2}} > 1.$$

When $k = i + 1$, again by Lemma~\ref{lem11}, we have
$$\frac{\pi_{i - 1}(A_k)}{\pi_i(A_k)} > \pi_{i - 1}(A_k) > \frac{1}{2(L + 1)^3}.$$

When $k = i + 2$, we have $\gamma^{1 - (i + 1 - k)^2} = 1$, so Eq~(\ref{eq01}) weakens to
$$\gamma^{1 - (i + 1 - k)^2 }+ \sum_{r = 0}^L \gamma^{-(i + 1 - r)(i - r)} < 3 + \frac{1}{(L + 1)^2},$$
while Eq~(\ref{eq02}) and Eq~(\ref{eq03}) continue to hold. Thus, we have
$$\frac{\gamma^{1 - (i - k)^2} + \sum_{r = 0}^L\gamma^{-(i - r)(i - 1 - r)}}{\gamma^{1 - (i + 1 - k)^2 }+ \sum_{r = 0}^L \gamma^{-(i + 1 - r)(i - r)}} > \frac{2}{3 + \frac{1}{(L + 1)^2}} > \frac{2}{3}\left(1 - \frac{1}{(L + 1)^2}\right),$$
and
$$\frac{\pi_{i - 1}(A_k)}{\pi_i(A_k)} > \frac{2}{3}\left(1 - \frac{1}{(L + 1)^2}\right)^2 > \frac{2}{3}\left(1 - \frac{1}{L + 1}\right).$$

Synthesizing the above results, we can conclude
$$B = \min_{k \in [m]}\prod_{i = 1}^L \min \left\{1, \frac{\pi_{i - 1}(A_k)}{\pi_i(A_k)}\right\} > \frac{2}{3}\left(1 - \frac{1}{L + 1}\right)^L\frac{1}{2(L + 1)^3} > \frac{1}{3e(L + 1)^3} > \frac{1}{(L + 1)^7},$$
where the last step is because $L + 1 \geq 2$. This finishes the proof of Lemma~\ref{sb}.
\end{proof}

Now we go back to prove Theorem~\ref{upper bound}.

\begin{proof}[Proof of Theorem \ref{upper bound}]
Let $\mathcal{X}_{\lambda}$ be the set of states that can be reached from $\lambda = (1, \ldots, m)$, i.e., $\lambda_i = i + 1$ for all $i \in \{0, \ldots, L\}$, by only swapping samples between adjacent levels, and $S$ be the subset of $\mathcal{X}_{\lambda}$ containing all the states that can be reached from $\lambda$ by only swapping samples between adjacent levels and with all intermediate states differing from $\lambda$ in at most $\lfloor\log L\rfloor - 1$ levels. Let $\overline{\pi}_{\lambda} = \frac{\overline{\pi}}{\overline{\pi}(\mathcal{X}_{\lambda})}$ be the density of the projected chain restricted to $\mathcal{X}_\lambda$. For $X \subset \mathcal{X}_\lambda$, let

$$\varphi(X) = \frac{\sum_{\xi \in X, \widetilde{\xi} \in X^C}\overline{\pi}_\lambda(\xi)\overline{P}_{pt}(\xi, \widetilde{\xi})}{\min\{\overline{\pi}_\lambda(X), \overline{\pi}_\lambda(X^C)\}},$$
where $X^C = \mathcal{X}_{\lambda} \setminus X$; let $\varphi^* = \min_{X : \overline{\pi}_\lambda(X) \leq \frac{1}{2}}\varphi(X)$. By Cheeger's inequality,
\begin{equation}
\label{eq5}
    \normalfont{\textbf{Gap}}(P_{pt}) \leq 2\varphi^* \leq 2\varphi(S),
\end{equation}
and we upper bound $\varphi(S)$ in the following. We bound each part of $\varphi(S)$ separately in the following lemmas.

\begin{lemma}
\label{lem14}
$\overline{\pi}_\lambda(S) > \frac{1}{2e}$.
\end{lemma}

\begin{proof}
By Lemma~\ref{lem11},
\[
\overline{\pi}_\lambda(S) \geq \overline{\pi}(\lambda) = \prod_{i = 0}^L \pi_i(A_i) > \left(1 - \frac{1}{L + 1}\right)^{L + 1} > \frac{1}{2e}.
\qedhere
\]
\end{proof}

\begin{lemma}
\label{lem15}
$\overline{\pi}_\lambda(\mathcal{X}_{\lambda}\setminus S) > \frac{1}{4e(L + 1)^6}$.
\end{lemma}

\begin{proof}
Let $\mathcal{P}$ be the set of all permutations of $\{2, \ldots, m - 1\}$. Let $\widetilde{S}$ be the set of states with $\lambda_0$ at the $L^{\text{th}}$ level and $\lambda_L$ at the $0^{\text{th}}$ level. By Lemma~\ref{lem8}, $\widetilde{S} \subset \mathcal{X}_\lambda\setminus S$, and thus
\begin{align*}
    \overline{\pi}_\lambda(S^C) \geq \overline{\pi}(S^C)  \geq \overline{\pi}(\widetilde{S}) &\geq \pi_0(A_m)\pi_L(A_1)\sum_{\sigma \in \mathcal{P}}\prod_{i = 1}^{L - 1}\pi_i(A_{\sigma(i + 1)})\\
    &> \left(\frac{1}{2(L + 1)^3}\right)^2\:\prod_{i = 1}^{L - 1}\pi_i(A_{i + 1})\\ &> \frac{1}{4(L + 1)^6} \left(1 - \frac{1}{L + 1}\right)^{L - 1} > \frac{1}{4e(L + 1)^6}.
\end{align*}
where the fourth and the fifth step are by Lemma~\ref{lem11}.
\end{proof}

\begin{lemma}
\label{lem13}
$\sum_{\xi \in S, \widetilde{\xi} \in S^C}\overline{\pi}_\lambda(\xi)\overline{P}_{pt}(\xi, \widetilde{\xi}) \leq 4e\left(\frac{1}{L + 1}\right)^{\lfloor \log L\rfloor - 2}$.
\end{lemma}

\begin{proof}
Let $S_1$ be the set of states with $\lfloor \log L\rfloor - 1$ levels out of place and $S_2$ be the set of states with $\lfloor \log L\rfloor - 2$ levels out of place. Observe that, in order for $\xi \in S$, $\widetilde{\xi} \in S^C$, and $\overline{P}_{pt}(\xi, \widetilde{\xi}) > 0$, we have either $\xi \in S_1$ or $\xi \in S_2$. Thus,
\begin{align}
\sum_{\xi \in S, \widetilde{\xi} \in S^C}\overline{\pi}_\lambda(\xi)\overline{P}_{pt}(\xi, \widetilde{\xi})
&= \sum_{\xi \in S_1, \widetilde{\xi} \in S^C}\overline{\pi}_\lambda(\xi)\overline{P}_{pt}(\xi, \widetilde{\xi}) + 
\sum_{\xi \in S_2, \widetilde{\xi} \in S^C}\overline{\pi}_\lambda(\xi)\overline{P}_{pt}(\xi, \widetilde{\xi})\\
&\leq \sum_{\xi \in S_1}\overline{\pi}_\lambda(\xi) + \sum_{\xi\in S_2}\overline{\pi}_\lambda(\xi).
\label{eq6}
\end{align}

For each $\xi \in S_1$, there are ${L + 1\choose \lfloor \log L\rfloor - 1}$ choices for which samples are out of place, and there are $(\lfloor \log L\rfloor - 1)!$ permutations of samples which are out of place. Also, by Lemma~\ref{lem11}, $\pi_i(A_k) < \frac{1}{(L + 1)^2}$ for all $i \in \{0, \ldots, L\}$ and $k \neq i + 1$. Therefore, we have

$$\sum_{\xi \in S_1}\overline{\pi}(\xi) \leq {L + 1\choose \lfloor \log L\rfloor - 1}\cdot (\lfloor \log L\rfloor - 1)! \cdot \left(\frac{1}{(L + 1)^2}\right)^{\lfloor \log L\rfloor - 1},$$
and because $\overline{\pi}(\mathcal{X}_\lambda) \geq \overline{\pi}(\lambda) > \frac{1}{2e}$,
\begin{align}
\sum_{\xi \in S_1}\overline{\pi}_\lambda(\xi) = \frac{1}{\overline{\pi}(\mathcal{X}_{\lambda})}\sum_{\xi \in S_1}\overline{\pi}(\xi) &< 2e {L + 1\choose \lfloor \log L\rfloor - 1}\cdot (\lfloor \log L\rfloor - 1)! \cdot \left(\frac{1}{(L + 1)^2}\right)^{\lfloor \log L\rfloor - 1}\\
&\leq 2e(L + 1)^{\lfloor \log L\rfloor - 1} \left(\frac{1}{(L + 1)^2}\right)^{\lfloor \log L\rfloor - 1} = 2e\left(\frac{1}{L + 1}\right)^{\lfloor \log L\rfloor - 1}.
\label{eq7}
\end{align}


Similarly, 
\begin{align}
\label{eq8}
    \sum_{\xi \in S_2}\overline{\pi}_\lambda(\xi) < 2e\left(\frac{1}{L + 1}\right)^{\lfloor \log L\rfloor - 2}.
\end{align}

Plugging Eq~(\ref{eq7}) and~(\ref{eq8}) back to Eq~(\ref{eq6}) gives

\[\sum_{\xi \in S, \widetilde{\xi} \in S^C}\overline{\pi}_\lambda(\xi)\overline{P}_{pt}(\xi, \widetilde{\xi}) < 4e\left(\frac{1}{L + 1}\right)^{\lfloor \log L\rfloor - 2}.\qedhere\]

\end{proof}

Combining Lemma~\ref{sb},~\ref{lem14},~\ref{lem15},~\ref{lem13} and Eq~(\ref{eq5}) gives
\begin{align*}
 \normalfont{\textbf{Gap}}(P_{pt}) \leq 2\varphi^* \leq 2\varphi(S)&= \frac{2\sum_{\xi \in S, \widetilde{\xi} \in S^C}\overline{\pi}_\lambda(\xi)\overline{P}_{pt}(\xi, \widetilde{\xi})}{\min\{\overline{\pi}_\lambda(S), \overline{\pi}_\lambda(S^C)\}}\\ 
 &< \frac{8e\left(\frac{1}{L + 1}\right)^{\lfloor \log L\rfloor - 2}}{\min\left\{\frac{1}{2e}, \frac{1}{4e(L + 1)^6}\right\}} = 32e^2 \left(\frac{1}{L + 1}\right)^{\lfloor\log L\rfloor - 8} < O\left( B^{O(\log L)}\right),
\end{align*}
which finishes the proof of Theorem~\ref{upper bound}.
\end{proof}

\section{Conclusion and Future Directions}
The main technical contribution of this paper is an improved lower bound on spectral gap that has a polynomial dependence on all parameters except $\log L$ for parallel tempering, together with a hypothetical upper bound. However, this is just a small step towards understanding the mixing time of these algorithms, and we present some open questions here. First, it's interesting to explore whether there exists a natural example in which we still have an upper bound on spectral gap that exponentially depends on $\log L$. Second, noticing that \cite{Simulated_Tempering} provides a lower bound on spectral gap for simulated tempering, which is another effective algorithm for sampling from multimodal distributions, that has a polynomial dependence on all parameters for a wide range of distributions, it's interesting to see whether there exists an instance where simulated tempering works well while parallel tempering doesn’t. Furthermore, one can investigate how our bounds can be applied to real distributions to obtain better guarantees on mixing time. Note that $\log L$ is only a lower bound, and the bound might be better under more assumptions.

\newpage
\bibliographystyle{plain}
\bibliography{ref}

\end{document}